\documentclass[preliminary]{eptcs}
\usepackage{breakurl}             	
\usepackage{underscore}           	

\usepackage{mathptmx}
\usepackage{amsmath} 			
\usepackage{amsthm}
\usepackage[english]{babel}
\usepackage{bm} 			
\usepackage{url} 			

\usepackage[dvipsnames]{xcolor}

\newcommand{\grnd}{\mathit{ground}}

\newtheorem{theorem}{Theorem}[section]

  \title{Automated Aggregator --- Rewriting with the Counting Aggregate}
  \author{Michael Dingess and Miroslaw Truszczynski\thanks{Research supported by National Science Foundation grant 1707371}
  \institute{Department of Computer Science, University of Kentucky, United States}
  \email{madingess@yahoo.com, mirek@cs.uky.edu}}

\def\titlerunning{Automated Aggregator --- Rewriting with the Counting Aggregate}
\def\authorrunning{Michael Dingess and Miroslaw Truszczynski}

\begin{document}

\label{firstpage}

\maketitle

\begin{abstract}
Answer set programming is a leading declarative constraint programming paradigm with wide use for complex knowledge-intensive applications. Modern answer set programming languages support many equivalent ways to model constraints and specifications in a program. However, so far answer set programming has failed to develop systematic methodologies for building representations that would uniformly 
lend well to automated processing. This suggests that encoding selection, in
the same way as algorithm selection and portfolio solving, may be a viable direction for improving performance of answer-set solving. The necessary precondition
is automating the process of generating possible alternative encodings. Here we present an automated rewriting system, the \emph{Automated Aggregator} or \emph{AAgg}, that given a non-ground logic program, produces a family of equivalent programs with complementary performance when run under modern answer set programming solvers. We demonstrate this behavior through experimental analysis and propose the system's use in automated answer set programming solver selection tools. 

\end{abstract}

\section{Introduction}
Developers of answer set programming (ASP) solutions often face situations 
where individual constraints of a problem or even the problem as a whole can
be expressed in several syntactically different but semantically equivalent 
ways. Picking the right representation is crucial to designing these solutions 
because, given an instance, certain representations perform better (often much
better) than others when processed with modern ASP grounders and solvers. 
However, techniques for selecting a particular representation are often
ad hoc and tailored to the needs of the particular application, and require
significant programming expertise from the programmer. 

About a decade ago, Gebser et al. presented a set of "rules-of-thumb" used by their
team in manual tuning of ASP solutions \cite{Gebser-2011}. These rules include suggestions on 
program rewritings that often result in substantial performance gains. This 
was verified experimentally by Gebser et al., with all rewritings used 
in their experiments generated manually \cite{Gebser-2011}. Buddenhagen and Lierler studied the 
impact of these
rewritings on an ASP-based natural language processor and reported orders of 
magnitude gains in memory and time consumption as a result of some program 
transformations they executed manually \cite{Buddenhagen-2015}. Because of these promising results,
researchers proposed to automate the task of program rewriting. 
Bichler et al. investigated rewritings of long rules
guided by the tree decomposition of graphs built of program 
rules (a form of join optimization) \cite{bichler-2016}. Hippen and Lierler proposed a method of applying projection to 
rewrite rules based on estimates of the size of the ground program \cite{Hippen-2018}.

\par
These projects demonstrated that automated program rewritings may lead to 
programs performing better than the original ones when run under current 
ASP solvers such as \emph{gringo/clasp} \cite{Gebser-2015}. However, the 
effects of rewritings are not uniformly the same. In fact, depending on the 
actual instance they are run with, rewritten programs may perform worse that 
the original ones. This is a problem because a non-uniform behavior makes
the process of selecting the \emph{uniformly best} encoding ill defined. 
Nevertheless, automated rewritings potentially can significantly improve
the state-of-the-art of ASP solving. Namely, the non-uniform behavior 
of programs obtained by rewriting opens a possibility of using families of 
equivalent alternative programs in algorithm selection and portfolio solving
\cite{rice-1976,XuHHL08,HoosLS14}. In fact, a recent study of suite of six 
encodings of the Hamiltonian cycle problem shows that one can train 
performance prediction models to select, given an instance, a
program from the suite to run on that instance,
guaranteeing a much better overall performance on large sets of instances 
than that of each of the six programs alone \cite{Liu-2019}.

\par
In this work, we focused on rewritings in which rules using counting based 
on explicit naming of a required number of objects are rewritten with the 
use of a \emph{counting} aggregate. We designed a software, \emph{Automated 
Aggregator} (\emph{AAgg}), for automating such rewritings into several 
equivalent forms. We studied the software's applicability and effectiveness. 
In our experiments, we applied the software to programs submitted to past 
ASP Competitions. We found that while many of them already included aggregate
expressions and \emph{AAgg} did not detect any rules to which it could apply, 
in several cases, it was applicable! For those cases, we studied the performance
of the original program and the rewritings produced by \emph{AAgg}. The results 
showed that depending on an instance, rewritings produced by our software often 
performed better than the original programs. In other words, the family of 
encodings generated by \emph{AAgg} (the original program and its rewritings) 
showed a complementary performance on both the instances used in the ASP 
Competitions and on instances which we generate ourselves. These results
show that \emph{AAgg} can be used as a tool for generating collections of 
encodings to be used in algorithm selection and portfolio solving. A systematic
experimental verification of this claim will be the subject of a future work.

\section{Aggregate Equivalence Rewriting}
In this section we describe the \emph{aggregate equivalence} rewriting, its 
input and output forms. Currently, we support one input form and three output
forms. The input form is a rule that expresses a constraint that there are $b$ 
different objects with a certain property by explicitly introducing $b$ 
variables to name these $b$ objects. The output forms model the same property 
by relying, in some way, on the counting aggregate. For each of the rewritings 
we establish its correctness and experimentally study its performance. The 
correctness follows from Theorems \ref{thm-main1} and
\ref{thm-split-replace}, presented and proved in the appendix.

\subsection{Preliminaries}
We follow the \emph{ASP-Core-2 Input Language Format}~\cite{Calimeri-2020}. We 
consider rules of the form $head \leftarrow body.$ The \emph{head} may consist
of a single literal or be empty, the latter representing a contradiction, 
making the rule a \emph{constraint}. The \emph{body} may contain one or more 
literals or be empty, which constitutes a \emph{fact}. \emph{Literals} are
composed of an atom, which may be preceded by $not$. \emph{Negative}
literals include $not$; \emph{positive} literals do not. \emph{Atoms} have 
the form $p(t_1,\dots ,t_k)$, where $p$ is a predicate symbol of arity $k$ 
and each $t_i$ is a term, that is, a constant, variable, or expression 
of the form $f(t_1,\dots ,t_k)$ where $f$ is a function symbol of arity 
$k>0$ and each $t_i$ is a term.\footnote{\emph{AAgg} also accepts other
term expressions following the definition used by \emph{clingo}
\cite{Gebser-2015}.}

\par
Atoms may also take the form of an aggregate expression. In this work, we focus
on \emph{counting} aggregates, that is, expressions of the form:
\begin{equation}
\label{eq:count-aggr}
	s_1 \prec_1 \#count \{ {\it\bf t}_1:{\it\bf L}_1{\bf ;}\ldots{\bf ;}{\it\bf t}_n:{\it\bf L}_n \} \prec_2 s_2
\end{equation} 

In (\ref{eq:count-aggr}), ${\it\bf t}_i$ and ${\it\bf L}_i$ form an \emph{aggregate element}, which is a non-empty tuple of terms and literals, respectively. 
The \emph{count} operation simply counts the number of \emph{unique} term tuples ${\it\bf t}_i$ whose corresponding condition ${\it\bf L}_i$ holds. The result of the count function is compared by the \emph{comparison predicates} $\prec_1$ and $\prec_2$ to the terms $s_1$ and $s_2$. These comparison predicates may be one of $\{<,\leq,=,\neq\}$. One or both of these comparisons can be omitted \cite{Gebser-2015}.

\subsection{Input Forms}
\label{section:input-forms}
The aggregate equivalence rewriting takes as input rules of the form:
\begin{equation}
\label{eq:input-form}
	H \leftarrow \bigwedge_{1\leq i\leq b} F(X_i,{\bf Y}), \bigwedge_{1\leq i<j\leq b} X_i \neq X_j, \ G.
\end{equation}
where 
\begin{itemize}
	\item $H$ is the head of a rule ($H$ may be empty making the rule a constraint)
	\item $F$ is a predicate of arity $1 + \bf{|Y|}$
	\item $X_1,\ldots, X_b$ are variables, all in the same position in $F$
	\item $\bf{Y}$ is a comma-separated list of variables, identical in variables and variable positions for all $F$ in the rule
	\item $G$ is the remaining body of the rule, possibly empty,
\end{itemize}
and the following hold true:
\begin{itemize}
	\item $b \geq 2$
	\item $H$, $G$, and ${\bf Y}$ have no occurrences of variables $X_1,
\ldots, X_b$
	\item The terms $\bigwedge_{1\leq i < j\leq b} X_i \neq X_j$ may instead be a continuous chain of comparisons:
	\begin{center}
	 $\bigwedge_{1 \leq i \leq b-1} X_i < X_{i+1}$ or $\bigwedge_{1 \leq i \leq b-1} X_i > X_{i+1}$.
	\end{center}
\end{itemize}

Note that $X_i$'s need not be in the first position in $F$, so long as they are all in the same position in $F$ and the other variables in $F$ (if any) are 
identical in all occurrences of $F$ in the rule. 
Additionally, some other forms logically equivalent to $\bigwedge_{1\leq i <
j\leq b} X_i \neq X_j$ are also acceptable. For instance, the condition 
$X_1 \neq X_2$ may be expressed as $X_1+a \neq X_2+a$, for some integer $a$.

\subsection{Output Forms}
The form of the output depends on the size of ${\bf Y}$.
When $|{\bf Y}| = 0$, the output form is:
\begin{equation}
\label{eq:output-form-lt}
	H \leftarrow  b \leq \#count\{X:F(X)\}, G.
\end{equation}
where 
\begin{itemize}
	\item $H$, $b$, $F$, and $G$ are as above
	\item Aggregate element $X:F(X)$ follows the form ${\it\bf term}:{\it\bf literal}$ as defined above. The prior, $X$, is a tuple of one term, which in this case is a variable. The second, $F(X)$, is a literal.
\end{itemize}

When $|{\bf Y}| > 0$, we perform \emph{projection} to project out the variable
$X$ from $F$. This gives us an output form consisting of two rules:
\begin{align}
\label{eq:output-form-lt-proj}
\begin{split}
        H  & \leftarrow  b \leq \#count\{X:F(X,{\bf Y})\}, F'({\bf Y}), G.\\
        F'({\bf Y}) & \leftarrow F(X,{\bf Y}).
\end{split}
\end{align}
where 
\begin{itemize}
	\item $H$, $b$, $F$, ${\bf Y}$, and $G$ are as above
	\item Aggregate element $X:F(X,{\bf Y})$ follows the form ${\it\bf term}:{\it\bf literal}$ as above
	\item $F'$ is a new predicate symbol of arity equal to the size of 
{\bf Y}, that is, equal to the arity of $F$ minus one; introducing $F'$
ensures that variables in $\mathbf{Y}$ are universally quantified.
\end{itemize}

The correctness of this rewriting, follows from the results presented in 
the appendix (Theorem \ref{thm-main1}). 
Specifically, we show there that if a program is obtained from another 
program by rewriting one of its rules in the way described above, then 
both programs have the same answer sets (modulo atoms $F'(\mathbf{y})$,
if $F'$ is introduced).

\subsubsection{Alternative Output Forms}
\par
Two alternative, logically equivalent output forms are also available, each derived from the output form presented above.
First, we observe that $b \leq F$ is logically equivalent to the negation of $F < b$. We can then restate the original literal as follows:
\begin{equation}
\label{eq:output-form-gt}
	not\ \#count\{X:F(X,{\bf Y})\} < b.
\end{equation}

Second, we note that the input language we consider permits integer-only 
arithmetic. Consequently, the expression $not\ a < b$ for some integers $a$ 
and $b$ is equivalent to the conjunctive expression:
\begin{center}
	$\neg(a=-\infty) \wedge \neg(a=-\infty+1) \wedge \ldots \wedge \neg(a=b-2) \wedge \neg(a=b-1)$.
\end{center}
Additionally, the result of the count never returns a negative number. Therefore, we can restate the aggregate literal in this alternative output form as a conjunction of aggregate literals having the form:
\begin{equation}
\label{eq:output-form-eq}
\begin{split}
	&not\ \#count\{X:F(X,{\bf Y})\} = 0, \\
	&not\ \#count\{X:F(X,{\bf Y})\} = 1, \\
	&\ldots, \\
	&not\ \#count\{X:F(X,{\bf Y})\} = b-1.
\end{split}
\end{equation}
Note that, due to the precise semantics of logic programs, the equivalence 
of these two alternative logic forms relies on additional assumptions about the input program (it has to be splittable) and the rule to be rewritten.
Theorem \ref{thm-split-replace} provides conditions under which the rewriting 
is guaranteed to be correct. These conditions are checked by our software and
only when they hold, the software proceeds with rewriting into an alternative
form (\ref{eq:output-form-gt}) or (\ref{eq:output-form-eq}), as selected by 
the user.

\section{The Automated Aggregator System}
We now present the \emph{Automated Aggregator}\footnote{The system and all encodings, test instances, and driver programs can be found online at:\\\url{https://drive.google.com/drive/folders/1lqRsy9HGIDHvyX_Pvkc8zKt1-xvbwcAp?usp=sharing}} (\emph{AAgg}) software system for performing the Aggregate Equivalence rewriting. The software provides an automated way to detect rules within a given program following the input format (\ref{eq:input-form}) and rewrite those rules into an equivalent output format (\ref{eq:output-form-lt}/\ref{eq:output-form-lt-proj}), (\ref{eq:output-form-gt}), or (\ref{eq:output-form-eq}).

\subsection{Usage}
The software relies on the \emph{clingo} Python module provided by the Potassco 
suite \cite{Gebser-2017}. The module is written in Python 2.7. As such, Python 
2.7 is required to run the \emph{Automated Aggregator} system. Installation 
information is provided in the software's README file. The \emph{Automated 
Aggregator} is invoked as follows: 

\begin{verbatim}
  python  aagg/main.py
          [-h,--help] [-o,--output FILENAME] [--no-rewrite] 
          [--no-prompt] [--use-anonymous-variable] 
          [--aggregate-form ID] [-d,--debug] [-r,--run-clingo]
          [encoding_1 encoding_2 ...]
\end{verbatim}

The {\tt -h} flag lists the help options. The encoding(s) are the filename(s) 
of the input encoding(s), and the output is the desired name for the output 
file. If no output filename is given, one is generated based on the first input 
filename given. When a candidate rule is discovered, the user is shown the 
proposed rewriting and prompted for confirmation. If the {\tt --no-rewrite} flag
is given, no prompts are given and no rewriting is performed. If the 
{\tt --no-prompt} flag is given, no prompts are given and rewriting is 
performed where possible.
 The {\tt ID} supplied to the {\tt --aggregate-form} argument informs the program
which aggregate form to use when performing rewrites: its values $1$, $2$, and
$3$ correspond to aggregate forms (\ref{eq:output-form-lt}/\ref{eq:output-form-lt-proj}), (\ref{eq:output-form-gt}), and (\ref{eq:output-form-eq}), respectively. The 
{\tt -d} debug flag directs the application to operate with verbosity, printing
details during the rewriting candidate discovery process and printing some 
statistics after the application's conclusion. The {\tt --r run-clingo} flag 
directs the application to run the resulting program through \emph{clingo} after any 
rewritings are performed. 

Finally, the {\tt --use-anonymous-variable} flag indicates an additional 
modification of the output form to be performed. It uses the anonymous variable 
`\_' in place of the variable $X$ in the output forms listed above, with some
additional modifications of the rule to ensure the transformation is correct. 
Specifically, we replace the aggregate element $X: F(X,{\bf Y})$ with 
$F(\_,{\bf Y}): F(\_,{\bf Y})$ rather than with $\_: F(\_,{\bf Y})$, 
because the latter is not a valid \emph{gringo} syntax. We mention this option
for completeness sake, since it is available in our implementation. However, 
we found that the programs
generated when using and when not using the anonymous variable are identical 
after grounding, so we neither discuss it further nor use these rewritings 
in our experiments.

By default all boolean flags are disabled and the aggregate-form {\tt ID} is 
set to $1$ indicating output form (\ref{eq:output-form-lt}). At least one input encoding filename must be 
specified.

\subsection{Methodology}
The methodology used for discovering whether a rule is a candidate for the aggregate equivalence rewriting is as follows. The given logic program(s) are parsed by the \emph{clingo} module, generating an abstract syntax tree for each rule. Each such tree is passed to a \emph{transformer} class for preprocessing. After preprocessing, some information is gathered from the program as a whole; specifically, predicate dependencies are determined, which in turn determine when output forms (\ref{eq:output-form-gt}) and (\ref{eq:output-form-eq}) are appropriate for a given rule (see the appendix for more details). Rules are then passed individually to an \emph{equivalence transformer} class for processing. After processing, and if the requested rewriting is possible and confirmed by the user, the rewritten form of the rule is returned. Otherwise the original rule is returned. All returned rules, rewritten or not, are collected and output to the desired output file location. Optionally, the resulting program is then run using \emph{clingo}.
\par
When a rule is passed to the equivalence transformer for processing, 
it first undergoes a process of exploration, which traverses the rule's 
abstract syntax tree, recording comparison literals between two variables as 
well as other pertinent information found along the way. These comparison literals
are scrutinized to determine whether a subset of the comparisons 
follows the form given in (\ref{eq:input-form}) or some equivalent format as detailed 
in the section \ref{section:input-forms}. The process also identifies those variables that
play the role of $X_1, \ldots, X_b$ as in (\ref{eq:input-form}). The rule is then 
analyzed to determine whether there are $b$ occurrences of some positive literal 
of predicate $F$ with the arguments $X_i$ and $\bf{Y}$, where each $X_i$,
$1\leq i\leq b$, exists at least once in the set of occurrences, and $\bf{Y}$ 
is the same for each occurrence and contains no variables $X_1, \ldots, X_b$.
Let us call this set of $b$ literals combined with the corresponding 
comparisons following the form given in (\ref{eq:input-form}) or equivalent, the rule's 
\emph{counting literals}. 
Similarly, we denote the variables playing the role of the variables $X_i$ as given in (\ref{eq:input-form}) as the rule's 
\emph{counting variables}.

After gathering counting literals and counting variables, the equivalence 
transformer verifies that the counting variables are not used within
literals anywhere in the rule excluding literals within the set of counting 
literals. If this verification fails or if any of the constraints for the 
counting literals cannot be satisfied or if no such set of counting variables 
can be obtained, then the rule is not fit for rewriting. As a result, no 
rewriting is performed on the rule and the original rule is returned.

In the other case, if the verification succeeds (the counting literals 
and the counting variables satisfy all the required constraints), then we 
proceed as follows. If the requested output form is (\ref{eq:output-form-lt}/\ref{eq:output-form-lt-proj}), then we check whether $|\mathbf{Y}|>0$ and rewrite 
the rule accordingly using (\ref{eq:output-form-lt}) or 
(\ref{eq:output-form-lt-proj}). If the requested output form is 
(\ref{eq:output-form-gt}) or (\ref{eq:output-form-eq}), 
then the predicate dependency condition related to splitting must 
be verified (see the appendix for details). If it holds, then we check whether 
$|\mathbf{Y}|>0$ and rewrite the rule into the desired form 
(\ref{eq:output-form-gt}) or (\ref{eq:output-form-eq}), adding the 
second of the two rules in (\ref{eq:output-form-lt-proj}) when 
$|\mathbf{Y}|>0$.

Note that the user is first prompted for confirmation for each rewritten rule 
before the rule is added to the final program. If the user denies a rewriting,
then the original rule is used.

\subsection{Limitations}
\label{section:limitations}
Here we discuss the limitations of the \emph{Automated Aggregator} in its 
current form. 

\begin{enumerate}
\item The rewriting is one-directional. The software will only 
rewrite rules from the form (\ref{eq:input-form}) into rules of the forms (\ref{eq:output-form-lt}/\ref{eq:output-form-lt-proj}), (\ref{eq:output-form-gt}), 
and (\ref{eq:output-form-eq}). 
As it stands, the system will not rewrite rules given in any of the forms (\ref{eq:output-form-lt}/\ref{eq:output-form-lt-proj}), (\ref{eq:output-form-gt}), or (\ref{eq:output-form-eq}) into rules of any of other form.

\item In the cases when multiple rewritings are possible for a single rule, 
only one rewriting will be detected and performed. To illustrate, if the form 
given in (\ref{eq:input-form}) occurs twice in one rule over a disjoint set of variables and 
predicates, where both sets of counting literals consider the other set as 
part of $G$ and all conditions hold, then only the set of counting literals 
with the highest number of counting variables will be used for rewriting. If both 
sets contain the same number of counting variables, then one is chosen based 
on the order in which the comparisons using those variables occur in the rule.

\item There are some cases in which the software will not recognize a valid rewriting when one exists. These cases mostly lie in the many obscure ways of representing a chain of comparisons equivalent to that in (\ref{eq:input-form}). However, there are no known cases in which an incorrect rewriting will be proposed when no valid rewriting is possible. More details are given in the software's README document. 

\end{enumerate}

\section{Experimental Analysis}
The \emph{Automated Aggregator} system has been made available for download 
online.\footnote{\url{https://drive.google.com/drive/folders/1lqRsy9HGIDHvyX_Pvkc8zKt1-xvbwcAp?usp=sharing}} Encodings with which the application was tested, their corresponding instances, and (Python) scripts for driving such tests, are included there too.

\subsection{Automated Aggregator in Practice}
The \emph{Automated Aggregator} system was applied to logic programs in 
\emph{gringo} syntax submitted to the 2009, 2014, and 2015 Answer Set 
Programming Competitions. Of the 58 encodings given to the application, 
five contained rules which were candidates for the rewriting described. 
Table \ref{asp-problem-domains-encodings-table} lists the encoding problem names and the ASP Competitions for which they were developed. 

\begin{table}
  \caption{Problem Domains of ASP Competition Encodings with Rules for Rewriting}\label{asp-problem-domains-encodings-table}
  \begin{minipage}{\textwidth}
    \begin{tabular}{ r | l }
      \hline
      {\bf ASP Competition Year } & {\bf Problem Domain of Encodings} \\
      \hline
      2009 & Golomb Ruler \& Wire Routing \\
      2014 & Steiner Tree \\
      2015 & Steiner Tree \& Graceful Graphs \\
      \hline
    \end{tabular}
  \end{minipage}
\end{table}

\par
Additionally, to provide an example of the functionality of \emph{AAgg}, the 
system was applied to an encoding for the Hamiltonian Cycle problem.
The original program is shown in Figure \ref{hamorig}. The rewritten program, as output by 
\emph{AAgg}, is shown in Figure \ref{hamrewr}. We see that the first two constraints in the program are both rewritten with the necessary projection performed as in 
(\ref{eq:output-form-lt-proj}). Experimental results for these encodings when applied to a generated set of hard instances are given in the following section.

\begin{figure}
\noindent\rule{\textwidth}{1pt}
\begin{center}
\begin{verbatim}
node(X) :- edge(X,Y).
node(X) :- edge(Y,X).

{ hc(X,Y) } :- edge(X,Y).
:- hc(X,Y), hc(X,Z), Y!=Z.
:- hc(X,Y), hc(Z,Y), X!=Z.

reach(X,Y) :- hc(X,Y).
reach(X,Y) :- hc(X,Z), reach(Z,Y).

:- node(X), node(Y), not reach(X,Y).
\end{verbatim}
\end{center}
\caption{Example Hamiltonian Cycle problem encoding. Original version.}\label{hamorig}
\noindent\rule{\textwidth}{1pt}
\end{figure}

\begin{figure}
\noindent\rule{\textwidth}{1pt}
\begin{center}
\begin{verbatim}
node(X) :- edge(X,Y).
node(X) :- edge(Y,X).

{ hc(X,Y) } :- edge(X,Y).
:- 2 <= #count{ Y : hc(X,Y) }, hc_project_Z(X).
hc_project_Z(X) :- hc(X,Y).
:- 2 <= #count { X : hc(X,Y) }, hc_project_Z1(Y).
hc_project_Z1(Y) :- hc(X,Y).

reach(X,Y) :- hc(X,Y).
reach(X,Y) :- hc(X,Z), reach(Z,Y).

:- node(X), node(Y), not reach(X,Y).
\end{verbatim}
\end{center}
\caption{Example Hamiltonian Cycle problem encoding. Rewritten version.}\label{hamrewr}
\noindent\rule{\textwidth}{1pt}
\end{figure}

\subsection{Experimental Results}
Results were gathered by systematically grounding and solving each 
instance-encoding pair within families of encodings for each problem type. 
The data sets of instances we used are available together with the 
\emph{AAgg} tool (see the url listed earlier).\footnote{We thank Daniel 
Houston and Liu Liu for providing us with the data sets of instances for 
the Latin Square and the Hamiltonian Cycle problems, respectively. The data
sets for the remaining problems were generated by software tools we developed
for the purpose. These tools and descriptions of instance sets used in 
experiments can be found at the \emph{AAgg} site.}
Each encoding in each family of encodings was run for each 
instance. The total grounding plus solving time of instance-encoding pairs were recorded and compared. 
\par
The precise encodings used as input to the \emph{AAgg} software for gathering results are listed in Table \ref{enc-sources}. Two output encodings are generated for each input encoding and together they form the encoding family for that problem domain. The two output forms used were those shown in equations (\ref{eq:output-form-lt}/\ref{eq:output-form-lt-proj}) and (\ref{eq:output-form-eq}); previous experiments showed extreme similarity between forms (\ref{eq:output-form-lt}/\ref{eq:output-form-lt-proj}) and (\ref{eq:output-form-gt}), so (\ref{eq:output-form-gt}) was left out to preserve machine time. The machine used for testing contained an Intel(R) Core(TM) i7-7700 CPU @ 3.60GHz with 16GB RAM.

\begin{table}
  \caption{Sources of Encodings Used for Experimental Results}\label{enc-sources}
  \label{sources-of-used-encodings-table}
  \begin{minipage}{\textwidth}
    \begin{tabular}{ r | l }
      \hline
      {\bf Source of Encoding } & {\bf Problem Domain of Encodings} \\
      \hline
      ASP Competition 2009 & Wire Routing \\
      ASP Competition 2015 & Steiner Tree \& Graceful Graphs \\
      Home-Brewed & Latin Squares\& Hamiltonian Cycle\footnote{We are thankful to Daniel Houston and Liu Liu for supplying the Latin Squares and Hamiltonian Cycle instance sets, respectively.} \\
      \hline
    \end{tabular}
  \end{minipage}
\end{table}

\begin{table}
  \caption{Result Statistics by Problem Domain}\label{results-stats}
  \label{Graceful Graphs}
  \begin{minipage}{\textwidth}
    \begin{tabular}{ c | r | r | r | r }
      \hline
      {\bf Encoding } & {\bf Wins} & {\bf Exclusive Wins} & {\bf Wins by 20\%} & {\bf Wins by 50\%} \\
      \hline
      \hline

      Wire Routing Input Encoding & 122 (58.4\%) & 34 (27.8\%) & 115 (92.0\%) & 97 (77.6\%) \\
      AAgg Output Form (\ref{eq:output-form-lt}/\ref{eq:output-form-lt-proj}) & 50 (24.0\%) & 26 (52.0\%) & 48 (87.3\%) & 38 (69.1\%) \\
      AAgg Output Form (\ref{eq:output-form-eq}) & 37 (17.7\%) & 13 (35.1\%) & 39 (90.7\%) & 23 (53.5\%) \\
      \hline

      Steiner Tree Input Encoding & 130 (33.9\%) & 0 (0\%) & 1 (0.8\%) & 0 (0\%) \\
      AAgg Output Form (\ref{eq:output-form-lt}/\ref{eq:output-form-lt-proj}) & 123 (32.1\%) & 0 (0\%) & 4 (3.3\%) & 0 (0\%) \\
      AAgg Output Form (\ref{eq:output-form-eq}) & 130 (33.9\%) & 0 (0\%) & 6 (4.6\%) & 0 (0\%) \\
      \hline

      Graceful Graphs Input Encoding & 212 (37.3\%) & 62 (29.2\%) & 182 (85.8\%) & 128 (60.4\%) \\
      AAgg Output Form (\ref{eq:output-form-lt}/\ref{eq:output-form-lt-proj}) & 97 (17.1\%) & 23 (23.7\%) & 82 (84.5\%) & 58 (59.8\%) \\
      AAgg Output Form (\ref{eq:output-form-eq}) & 259 (45.6\%) & 51 (19.7\%) & 229 (88.4\%) & 169 (63.7\%) \\
      \hline

      Latin Squares Input Encoding & 5611 (77.7\%) & 5 (0.1\%) & 4435 (79.0\%) & 2003 (35.7\%) \\
      AAgg Output Form (\ref{eq:output-form-lt}/\ref{eq:output-form-lt-proj}) & 1432 (19.8\%) & 0 (0\%) & 655 (45.7\%) & 86 (6.0\%) \\
      AAgg Output Form (\ref{eq:output-form-eq}) & 176 (2.4\%) & 0 (0\%) & 64 (36.4\%) & 7 (4.0\%) \\
      \hline

      Hamiltonian Cycle Input Encoding & 72 (28.7\%) & 47 (65.3\%) & 59 (81.9\%) & 34 (47.2\%) \\
      AAgg Output Form (\ref{eq:output-form-lt}/\ref{eq:output-form-lt-proj}) & 102 (40.6\%) & 69 (67.7\%) & 79 (77.5\%) & 52 (51.0\%) \\
      AAgg Output Form (\ref{eq:output-form-eq}) & 77 (30.7\%) & 45 (58.4\%) & 64 (83.1\%) & 44 (57.1\%) \\
      \hline

    \end{tabular}
  \end{minipage}
\end{table}

Result statistics are shown in Table \ref{results-stats}. Data is grouped by problem domain. The first line of each grouping shows statistics for the input encoding. The second line shows statistics for the encoding output by the \emph{AAgg}
software using the first line as input and the output rule form (\ref{eq:output-form-lt}/\ref{eq:output-form-lt-proj}) as the selected output form. The third and final line of each grouping shows statistics for the encoding output by the \emph{AAgg}
software again using the first line as input but now selecting the output form (\ref{eq:output-form-eq}) as the chosen output form. 

A win is when an encoding grounds and solves for an instance in the shortest amount of time as compared with the other two encoding in its grouping. The percentage value beside the win count is the proportion of instances which that encoding won relative to the number of instances in the instance set for which at least one of the three
encodings terminated within the total time limit set for grounding and solving. This time limit was set to 200 seconds in all experiments, except those with
encodings of the Hamiltonian Cycle problem that used a timeout value of 400 
seconds due to the relative 
difficulty of the instance set (in all problem domains except for the Hamiltonian Cycle domain, the number of 
instances for which no encoding computed an answer within the time limit was
very small; for the Hamiltonian Cycle problem, even with the increased time limit, 
it was significant). 

An exclusive win is when the encoding is the only encoding in its grouping to find a solution (or determine there is no solution) for an instance while both of the other two encodings failed to do so within the time limit. The column shows the number and the percentage of wins which were exclusive wins for the 
encoding. The next column shows the number and the percentage of wins by a
margin of at least 20\% (this is when the best encoding runs at least
20\% faster than the second one; in case the of exclusive wins, we count a
win as by at least 20\%, if it is faster by at least 20\% than the time limit 
used). The data in the last column shows the numbers and percentages of wins
by at least 50\% (it is to be interpreted similarly to the data in the previous
column).

\par
Results indicate that in some cases, the rewriting can provide complementary 
performance of encodings. Specifically, the results for the Wire Routing, 
Graceful Graphs, and Hamiltonian Cycle problems support the claim.
This is 
shown first by the fact that for each problem, each of the encodings scores
a significant proportion of wins, and that among those wins a significant
proportion are exclusive wins, and an even greater proportion (about 50\% 
or more) are wins by at least 50\%. This means that about half of the times 
when an encoding outperforms the two other encodings, it outperforms 
both by a factor of at least two. 

The Steiner Tree results indicate that sometimes the rewriting produces little 
to no effect at all.  While each encoding for the problem registers a similar 
proportion of wins, there are very few instances (just eleven out of 283) when 
the best encoding outperforms the other two by 20\% or more, and no instances
when the best encoding would outperform the other two by 50\% or more.
The Latin Squares results are mixed in the sense that for the most part the original encoding works best. However, one of the rewritings registers almost 20\%
of wins. Moreover, about 45\% of those wins are by 20\% or more and 6\% by 50\% or more.

In summary, we see that the rewriting can provide programs that perform 
complementary to the original. This complementarity is perhaps somewhat 
surprising, because aggregates are assumed to lead to better performance.
Our results show that the picture is more complicated and whether rewriting
with aggregates yields better performance is instance-dependent. It is also 
interesting to note that for some encodings replacing simple counting, like
that used in the Latin Square encodings (no two identical elements in a row 
or column), with the \emph{count} aggregates does not lead to substantial 
improvements. Finally, for some problems (in our experiments for the Steiner 
Tree problem) where complementary behavior does emerge, the differences in 
performance are relatively small.

\section{Future Work}
As detailed in Section \ref{section:limitations}, the Automated Aggregator software
can be improved. Extending the software so that it rewrites rules by 
eliminating the \emph{counting} aggregate by inverting the current 
rewriting seems to be a potentially most beneficial direction. Indeed, our 
results show that introducing aggregates does not guarantee uniformly 
improved performance. This gives reason to think that removing the 
\emph{counting} aggregate, that is, rewriting it in an explicit way, 
has a potential of generating encodings that on many instances may 
perform better. Just like the present version of the \emph{AAgg} system,
this could yield collections of encodings of complementary strengths. 
Moreover, this form of rewriting would be quite widely applicable, as 
the \emph{counting} aggregate is commonly used.

Expanding the software to detect and perform multiple aggregate equivalence 
rewritings on a single rule and to detect more obscure forms of representations are two other directions for improvement. While necessary to ensure the
software has a possibly broad scope of applicability, we do not expect these 
extensions to have a major practical impact due to low frequency with which
such less intuitive and convoluted forms of modeling are found in programs.

Automated rewriting has been studied before. The {\sc Projector} system 
\cite{Hippen-2018} and the \emph{lpopt} system \cite{bichler-2016} are
two notable examples. That earlier work sought to develop rewritings
improving the performance over the original ones. That goal is in general
difficult to meet; both systems were shown to offer gains, but the rewritten
encodings are not always performing better. Our goal was different. We aimed
at rewritings of varying relative performance depending on input instances.
With the \emph{AAgg} system we showed that even a rather simple rewriting 
consisting of introducing the \emph{count} aggregate often leads to families 
of encodings with complementary strengths (areas of superior performance). 
This opens a way for using machine learning to develop models in support 
of effective encoding selection or encoding portfolio solving \cite{Liu-2019}. 
In such work, to generate promising collections of encodings, one could use 
the \emph{AAgg} system, with extensions mentioned above, but also other 
program rewriting software such as {\sc Projector} and \emph{lpopt} which, as noted, 
while they do yield good encodings, often better than the original one, they do not perform 
uniformly better.

\bibliographystyle{eptcs}
\bibliography{AAggRef}

\appendix
\section{Correctness of Rewritings}

Let us consider the following program rule ($H$ may be $\bot$):
\begin{equation}
\label{eq3}
H \leftarrow \bigwedge_{1\leq i\leq b} F(X_i,\bm{Z},\bm{Z'}),
 Q(\bm{X}), G(\bm{Z'},\bm{Z''}).
\end{equation}
where $F$ is a predicate, $\bm{X}$ a tuple of variables $X_1,\ldots, X_b$,
$Q(\bm{X})$
is a list of literals over variables $X_1,\ldots, X_b$, and $\bm{Z}, 
\bm{Z'}$ and $\bm{Z''}$ are three pairwise disjoint tuples of variables, 
and disjoint with $\bm{X}$, and $H$ contains none of $X_i$.

Let $P$ be a program containing rule (\ref{eq3}) and let $P'$ be the program
obtained from $P$ by replacing that rule with the two rules
\begin{align}
\label{eq4}
\begin{split}
H &\leftarrow \bigwedge_{1\leq i\leq b} F(X_i,\bm{Z},\bm{Z'}), 
Q(\bm{X}), G(\bm{Z'},\bm{Z''}), F'(\bm{Z}).\\
F'(\bm{Z}) &\leftarrow F(X,\bm{Z},\bm{Z'}).
\end{split}
\end{align} 
where $F'$ is a predicate not occurring in $P$.

\begin{theorem}
\label{thm-add-domain}
The programs $P$ and $P'$ have the same answer sets modulo ground atoms of the
form $F'(\bm{z})$.
\end{theorem}
\begin{proof} (Sketch) Consider a ground instance $r$ of rule (\ref{eq3}) and
let $a$ be the first atoms in the body of $r$ (that is, a ground instance of 
the atom $F(X_1,\bm{Z},\bm{Z'})$). In $P'$ there are ground rules $r'$ and 
$r''$ 
obtained from (\ref{eq4}) using the same variable instantiation as that 
used to produce $r$. Clearly, $r$ contributes to the reduct of $\grnd(P)$
if and only if $r'$ and $r''$ contribute to the reduct of $\grnd(P')$.
Moreover, if they do, $r$ ``fires'' in the least model computation if and
only if $r'$ fires in the least model computation. It follows that an 
interpretation $I$ of $P$ is an answer set of $P$ if and only if $I\cup J$
is an answer set of $P'$, where $J$ consists of all atoms $F'(\bm{z})$ such 
that $F(x,\bm{z},\bm{z'}) \in I$, for some constant $x$ and a tuple of 
constants $\bm{z'}$.
\end{proof}

Next, we recall the following theorem proved by Lierler \cite{Lierler-2019}.

\begin{theorem}\label{yu-li}
Let $H$ be an atom ($H$ may be $\bot$), $G$ a list of literals, $X$ 
a variable, $\bm{Z}$ a tuple of variables, each different from $X$ and each
with at least one occurrence in a literal in $G$, and $F$ a predicate symbol 
of arity $1+|\bm{Z}|$. If $b$ is an integer, and $X_1,\ldots, X_b$ are 
variables without any occurrence in $H$ and $G$, then the logic program rule 
\begin{equation}
\label{eq2}
H \leftarrow \bigwedge_{1\leq i\leq b} F(X_i,\bm{Z}), \bigwedge_{1 \leq i < j 
\leq b} X_i\neq X_j, G.
\end{equation}
is strongly equivalent to the logic program rule
\begin{equation}
\label{eq1}
H \leftarrow  b \leq \#count\{X:F(X,\bm{Z})\}, G.
\end{equation}
where $\bigwedge$ is used to represent a sequence of expressions separated 
by commas.
\end{theorem}

Combining the two results leads to the following result that proves the
correctness of the first rewriting implemented by our tool \emph{AAgg}.

\begin{theorem}
\label{thm-main1}
Let $P$ be a program containing a rule $r$ of the form 
\begin{equation}
\label{eq6}
H \leftarrow \bigwedge_{1\leq i\leq b} F(X_i,\bm{Z},\bm{Z'}), \bigwedge_{1\leq i<j\leq b} X_i\neq X_j,\ \ G(\bm{Z'},\bm{Z''}).
\end{equation}
under the same assumptions about variable tuples $\bm{X}, \bm{Z}, \bm{Z'}$
and $\bm{Z''}$ as before. If $\bm{Z}$ is empty, let $P'$ be $P$. Otherwise,
let $P'$ be obtained from $P$ by replacing $r$ with the rules (\ref{eq4}),
adjusting the first of them to contain $\bigwedge_{1\leq i<j\leq b} X_i\neq 
X_j$ in place of $Q(\bm{X})$. Let us call the first of these two rules $r$ 
(reusing the name of the original rule). Finally, let $P''$ be obtained from 
$P'$ by replacing $r$ with the corresponding rule (\ref{eq1}). Then, the 
programs $P$ and $P'$ have the same answer sets (in the case, when $\bm{Z}$ 
is not empty, the same answer sets modulo atoms $F'(\bm{z})$).
\end{theorem}
\begin{proof}
Clearly, rule $r$ is a special case of a rule of the form (\ref{eq3}). By
\ref{thm-add-domain}, programs $P$ and $P'$ have the same answer 
sets (when $\bm{Z}$, is not empty, the same answer sets modulo atoms 
$F'(\bm{z})$). The rule $r$ in
$P'$ is of the form (\ref{eq2}) required by Theorem \ref{yu-li}. Applying this
theorem shows that programs $P'$ and $P''$ have the same answer sets and the 
assertion follows. \\ \mbox{\ }\hfill \end{proof}
 
Once a program has a rule of the form (\ref{eq1}), we can often modify it 
further by exploiting alternative encodings of the aggregate expressions. 
In particular, under some assumptions about the structure of the program, 
we can replace a rule (\ref{eq1}) by 
\begin{equation}
\label{eq7}
H \leftarrow not\: \#count\{X:F(X,\bm{Z})\} < b, G.
\end{equation}
where we assume that the variable $X$ is not a variable in $\bm{Z}$, and 
that all variables in $\bm{Z}$ appear in $G$.

We recall that a partition $(P_b,P_t)$ of a program $P$ is a \emph{splitting} 
of $P$ if no predicate appearing in the head of a rule from $P_t$ appears in
$P_b$ \cite{Lifschitz-1994,Ferraris-2009}. A well-known result on splitting
states that answer set of programs that have a splitting can be described 
in terms of answer sets of programs that form the splitting.

\begin{theorem}
\label{thm-split}
Let $P$ be a logic program and let $(P_b,P_t)$ be a splitting of $P$.
For every answer set $I_b$ of $P_b$, every answer set of the program
$P_t\cup I_b$ is an answer set of $P$. Conversely, for every answer set
$I$ of $P$, there is an answer set $I_b$ of $P_b$ such that $I$ is
an answer set of $P_t\cup I_b$. 
\end{theorem}

This result implies that in programs that have a splitting, a rule
in $P_t$ containing in its body an aggregate expression involving only
predicates appearing in $P_b$ can be replaced by a rule in which this
aggregate expression is replaced by any of its (classically) equivalent 
forms. We formally state this result for the case of rules of the form
(\ref{eq1}).

\begin{theorem}
\label{thm-split-replace}
Let $P$ be a logic program and let $(P_b,P_t)$ be a splitting of $P$.
If $P_t$ contains a rule of the form (\ref{eq1}) and $F$ appears in $P_b$,
then $P$ and the program $P'$ obtained from $P$ by replacing the rule 
(\ref{eq1}) by the rule (\ref{eq7}) have the same answer sets.
\end{theorem}
\begin{proof} (Sketch)
Let $P_t'$ be the program obtained from $P_t$ by replacing the rule (\ref{eq1})
by the rule (\ref{eq7}). It is clear that $(P_b,P_t')$ is a splitting of $P'$. 
Let $I$ be an answer set of $P$. By \ref{thm-split}, there is an answer 
set $I_b$ of $P_b$ such that $I$ is an answer set of $P_t \cup I_b$. In 
particular, $I$ is an answer set of the program $I_b\cup \grnd(P_t)$. Let $Q$ 
be the program obtained by simplifying the bodies of the rules in $\grnd(P_t)$ 
as follows. If a conjunct $c$ in the body of a rule in $\grnd(P_t)$ involves 
only atoms from the Herbrand base $HB(P_b)$ of $P_b$, we remove $c$ if 
$I_b\models c$, and we remove the rule, if $I_b\not\models c$. Because atoms
from $HB(P_b)$ do not appear in the heads of the rules in $\grnd(P_t)$, 
$I$ is an answer set of $Q\cup I_b$.

We denote by $Q'$ the program obtained by the same simplification process 
from $\grnd(P_t')$. From the definition of $P_t'$ it follows that $Q'=Q$
(indeed, the only difference between $P_t$ and $P_t'$ is in the bodies of 
some rules, in which an aggregate built entirely from the atoms in $HB(P_b)$ 
is replaced by a classically equivalent one; thus, both expressions evaluate 
in the same way under $I_b$ and the contribution of the corresponding rules
to $Q$ and $Q'$ in each case is the same). Consequently, $I$ is an answer 
set of $Q'\cup I_b$. Because atoms from $HB(P_b)$ do not appear in the heads 
of the rules in $\grnd(P_t')$, $I$ is an answer set of $\grnd(P_t')\cup I_b$ 
and, because $(P_b,P_t')$ is a splitting of $P'$, also an answer set of $P'$. 
A similar argument shows that answer sets of $P'$ are also answer sets of $P$.
\end{proof}

It is clear that the same argument applies to other similar rewritings, for 
instance, to the one that replaces the rule (\ref{eq2}) by the rule
\begin{equation}
\label{eq8}
H \leftarrow \bigwedge_{0 \leq i < b - 1} not\: i = \#count\{X:F(X,\bm{Z})\}, G.
\end{equation}
where, as before, we assume that the variable $X$ is not a variable in 
$\bm{Z}$, and that all variables in $\bm{Z}$ appear in $G$. 

Theorem \ref{thm-split-replace} establishes the correctness of the two remaining
rewritings implemented by the tool \emph{AAgg}. The tool is designed to check
for splitting and allows the rewriting to take place only if the conditions
of \ref{thm-split-replace} hold.

\label{lastpage}
\end{document}